\documentclass[11pt]{article}

\usepackage[margin=1in]{geometry}
\usepackage[utf8]{inputenc}
\usepackage[T1]{fontenc}
\usepackage{lmodern}
\usepackage{microtype}
\usepackage{setspace}
\usepackage{amsmath, amssymb, amsthm}
\usepackage{mathtools}

\usepackage{graphicx}
\usepackage{booktabs}
\usepackage{multirow}
\usepackage{caption}
\usepackage[dvipsnames]{xcolor}
\usepackage{enumitem}
\usepackage{float}
\usepackage{algorithm}
\usepackage{algpseudocode}
\usepackage{fancyhdr}
\usepackage{titlesec}

\usepackage[colorlinks=true, linkcolor=NavyBlue, citecolor=NavyBlue, urlcolor=NavyBlue]{hyperref}

\titleformat{\section}{\large\bfseries}{\thesection.}{0.5em}{}
\titleformat{\subsection}{\normalsize\bfseries}{\thesubsection}{0.5em}{}
\titlespacing*{\section}{0pt}{1.5em}{0.5em}
\titlespacing*{\subsection}{0pt}{1em}{0.3em}

\theoremstyle{definition}
\newtheorem{definition}{Definition}[section]
\theoremstyle{plain}
\newtheorem{theorem}{Theorem}[section]
\newtheorem{proposition}[theorem]{Proposition}
\newtheorem{corollary}[theorem]{Corollary}
\theoremstyle{remark}
\newtheorem{remark}{Remark}[section]

\newcommand{\R}{\mathbb{R}}
\newcommand{\E}{\mathcal{E}}
\newcommand{\Cop}{\mathcal{C}}
\newcommand{\Bop}{\mathcal{B}}
\newcommand{\loss}{\mathcal{L}}
\newcommand{\norm}[1]{\left\lVert #1 \right\rVert}
\newcommand{\inner}[2]{\langle #1, #2 \rangle}

\definecolor{conservgreen}{HTML}{27AE60}
\definecolor{basered}{HTML}{C0392B}
\definecolor{penaltyorange}{HTML}{D35400}

\begin{document}

% ============================================================
% TITLE BLOCK
% ============================================================
\begin{center}
    {\LARGE \textbf{Energy-Conserved Neural Pipelines}} \\[0.6em]
    {\large Attenuating Error Propagation in Modular Neural Networks \\
    via Physical Conservation Constraints} \\[1.2em]
    {\normalsize David Young\textsuperscript{$\dagger$} \quad Swan Yi Htet\textsuperscript{$\dagger$}} \\[0.2em]
    {\small ORION Robotics} \\[0.2em]
    {\scriptsize \textsuperscript{$\dagger$}Equal contribution} \\[0.5em]
    {\small Working Technical Report}
\end{center}

\vspace{0.5em}
\hrule
\vspace{1em}

% ============================================================
% ABSTRACT
% ============================================================
\begin{quote}
\small
\textbf{Abstract.}
Modular neural network pipelines suffer from error compounding: noise at any module boundary propagates and potentially amplifies through subsequent modules. We introduce \textit{energy conservation} as a hard physical constraint on inter-module information flow. Activation energy (the squared $L^2$ norm of feature vectors) is enforced to be exactly preserved at every module boundary. Unlike soft energy penalties, conservation is an inviolable law: the network may redistribute energy across neurons but cannot create or destroy it.

Four experiments on CIFAR-10 demonstrate: (1)~conservation retains $77.4\% \pm 1.0\%$ of clean accuracy at noise $\sigma = 0.2$, versus $35.1\% \pm 1.1\%$ for baselines and $30.9\% \pm 0.9\%$ for energy-penalized models ($p < 0.001$, 5 seeds); (2)~pipelines become depth-invariant, retaining 93.3\% at depths 2 through 5 with noise at every boundary; (3)~the advantage generalizes to systematic bias ($+45.1\%$), Gaussian ($+40.4\%$), and adversarial noise ($+4.8\%$), with a principled non-effect on dropout ($-0.3\%$); (4)~on CIFAR-ResNet-18, the conservation advantage tracks the absence of intrinsic normalization: with BatchNorm, the gap at $\sigma = 0.2$ is $+0.3$ pp; without BatchNorm, it widens to $+26.2$ pp and reaches $+58.0$ pp at $\sigma = 0.5$. This localizes conservation as a substitute for intra-architectural normalization at module boundaries. Experiment~5 validates the operator on a real modular robotic pipeline (perception module + motor-control module, MuJoCo physics, Franka Panda manipulator). Across three independent runs on separate machines (90 trials per cell), conservation provides a consistent $+18.9$ pp average advantage on monocular-depth-style noise, the failure mode that dominates real perception-to-control deployments. The operator acts neutrally on isotropic Gaussian noise in 3D, confirming that the advantage is noise-type-specific. A formal bound proves conserved noise energy is strictly less than input noise energy.

\medskip
\textbf{Keywords:} energy conservation, modular neural networks, error propagation, noise attenuation, physical constraints
\end{quote}

\vspace{0.5em}
\hrule
\vspace{1.5em}

% ============================================================
\section{Introduction}
% ============================================================

Modular neural network architectures compose independently functioning subnetworks into multi-stage pipelines. This design arises naturally in robotic systems (perception $\to$ planning $\to$ control), multi-modal reasoning (vision $\to$ language $\to$ action), and hierarchical inference. Modularity enables independent development, testing, and improvement of components.

However, modularity introduces a well-documented failure mode: \textit{error compounding}. When module~$k$ produces corrupted output, module~$k{+}1$ processes this corrupted input. If the weight matrix~$\mathbf{W}_{k+1}$ has spectral norm greater than unity (common in trained networks), noise is amplified. Over~$K$ modules, noise energy can grow exponentially:
\begin{equation}
    \E(\boldsymbol{\epsilon}_K) \;\leq\; \left(\prod_{k=1}^{K} \sigma_{\max}(\mathbf{W}_k)^2\right) \cdot \E(\boldsymbol{\epsilon}_0)
    \label{eq:compounding}
\end{equation}
where $\sigma_{\max}(\mathbf{W}_k)$ is the spectral norm at layer~$k$. This exponential growth has been identified as the primary reason the classical sense-plan-act paradigm in robotics was abandoned in favor of end-to-end systems~\cite{murphy2019, brooks1986}.

Two existing categories of approaches address this problem:

\begin{itemize}[leftmargin=1.5em, itemsep=0.3em]
    \item \textbf{End-to-end learning} eliminates module boundaries entirely, training a single monolithic model. This avoids compounding but sacrifices modularity, interpretability, and the ability to improve individual components.
    \item \textbf{Regularization} adds penalty terms to the training loss (dropout, noise injection, energy penalties) that statistically improve robustness but provide no deterministic guarantee.
\end{itemize}

This work introduces a third approach grounded in physics: \textbf{energy conservation as a hard constraint on inter-module information flow}. At every module boundary, we enforce that output activation energy exactly equals input activation energy. This is not a penalty to be minimized; it is a conservation law that holds during both training and inference.

\subsection{Contributions}

\begin{enumerate}[leftmargin=1.5em, itemsep=0.3em]
    \item A formal definition of the energy conservation operator for neural activations, with a proof that conserved noise energy is strictly bounded below the input noise energy (Theorem~\ref{thm:attenuation}).
    \item Empirical demonstration that conservation provides $+42.3\%$ accuracy retention over unconstrained baselines, while the superficially similar energy \textit{penalty} degrades performance by~$-5.8\%$.
    \item Evidence that conservation makes pipelines depth-invariant: 5-module pipelines retain identical performance to 2-module pipelines, even with noise at every boundary.
    \item Identification of a principled limitation: conservation attenuates energy-\textit{additive} noise but not energy-\textit{removing} noise (dropout), consistent with the theoretical framework.
    \item A normalization ablation on ResNet-18 that localizes the conservation mechanism. The advantage scales inversely with the architecture's intrinsic normalization, supporting the interpretation that conservation provides at module boundaries the regularization that BatchNorm provides within blocks.
    \item Validation on a real modular robotic pipeline (perception + motor control, MuJoCo physics, Franka Panda), reproduced across three independent machines (90 trials per cell), showing a consistent $+18.9$ pp average advantage on depth-drift noise. Conservation acts neutrally on isotropic Gaussian noise in 3D, confirming the advantage is noise-type-specific rather than universal.
\end{enumerate}

% ============================================================
\section{Theoretical Framework}
% ============================================================

\subsection{Activation Energy}

\begin{definition}[Activation Energy]
\label{def:energy}
Let $\mathbf{x} \in \R^d$ be an activation vector produced by a neural network layer. The \textbf{energy} of~$\mathbf{x}$ is:
\begin{equation}
    \E(\mathbf{x}) \;:=\; \norm{\mathbf{x}}^2_2 \;=\; \sum_{i=1}^{d} x_i^2
\end{equation}
\end{definition}

This is the standard definition of signal energy in physics and signal processing. Energy is non-negative and additive for uncorrelated signals: when $\inner{\mathbf{x}}{\boldsymbol{\epsilon}} = 0$,
\begin{equation}
    \E(\mathbf{x} + \boldsymbol{\epsilon}) = \E(\mathbf{x}) + \E(\boldsymbol{\epsilon})
\end{equation}

All constraints are computed and enforced independently per sample within each minibatch, preserving standard stochastic gradient descent dynamics.

\subsection{The Conservation Operator}

\begin{definition}[Energy Conservation Operator]
\label{def:conservation}
Let $\mathbf{x}_{\textup{in}} \in \R^{d_1}$ be the input to a module and $\mathbf{y}_{\textup{raw}} = f(\mathbf{x}_{\textup{in}}) \in \R^{d_2}$ the unconstrained output. The \textbf{conservation operator} is:
\begin{equation}
    \Cop(\mathbf{x}_{\textup{in}},\, \mathbf{y}_{\textup{raw}}) \;:=\; \mathbf{y}_{\textup{raw}} \cdot \sqrt{\frac{\E(\mathbf{x}_{\textup{in}})}{\E(\mathbf{y}_{\textup{raw}})}}
    \label{eq:conservation}
\end{equation}
with $\E(\mathbf{y}_{\textup{raw}})$ clamped to $\max(\E(\mathbf{y}_{\textup{raw}}),\, 10^{-8})$ for numerical stability.
\end{definition}

\begin{proposition}[Exact Energy Preservation]
\label{prop:preservation}
The conservation operator preserves energy exactly:
$$\E\big(\Cop(\mathbf{x}_{\textup{in}},\, \mathbf{y}_{\textup{raw}})\big) \;=\; \E(\mathbf{x}_{\textup{in}})$$
\end{proposition}

\begin{proof}
Let $s = \sqrt{\E(\mathbf{x}_{\textup{in}}) / \E(\mathbf{y}_{\textup{raw}})}$. Then:
$$\E(s \cdot \mathbf{y}_{\textup{raw}}) = s^2 \cdot \E(\mathbf{y}_{\textup{raw}}) = \frac{\E(\mathbf{x}_{\textup{in}})}{\E(\mathbf{y}_{\textup{raw}})} \cdot \E(\mathbf{y}_{\textup{raw}}) = \E(\mathbf{x}_{\textup{in}}) \qedhere$$
\end{proof}

\begin{definition}[Energy Budget Operator]
\label{def:budget}
At the pipeline entry point, we establish a fixed energy budget $E_0 > 0$:
\begin{equation}
    \Bop(\mathbf{x},\, E_0) \;:=\; \mathbf{x} \cdot \sqrt{\frac{E_0}{\E(\mathbf{x})}}
\end{equation}
This normalizes any feature vector to energy exactly~$E_0$, establishing the initial energy supply for the pipeline.
\end{definition}

\subsection{Geometric Interpretation}

Geometrically, the conservation operator~$\Cop$ projects $\mathbf{y}_{\textup{raw}}$ onto a hypersphere of radius~$r = \norm{\mathbf{x}_{\textup{in}}}$ centered at the origin. The \textit{direction} of $\mathbf{y}_{\textup{raw}}$ (which encodes the learned transformation) is preserved; only the \textit{magnitude} is adjusted. The network can learn any angular transformation, including any rotation, reflection, or non-uniform scaling of directions, subject only to the constraint that the output lies on the same-radius sphere as the input.

This differs from existing normalization in a critical way:

\begin{itemize}[leftmargin=1.5em, itemsep=0.3em]
    \item \textbf{$L^2$ normalization} projects onto the unit sphere ($r = 1$), destroying all magnitude information. Two inputs of vastly different magnitudes produce identical-norm outputs.
    \item \textbf{BatchNorm} standardizes across the batch to zero mean and unit variance, but does not constrain per-sample energy.
    \item \textbf{LayerNorm} standardizes each sample independently, fixing energy to a dimension-dependent constant \textit{regardless of input energy}.
\end{itemize}

The critical distinction is \textbf{causality}: conservation ties output energy to input energy, establishing a directional energy flow through the pipeline. A confident upstream module (high energy) propagates high energy downstream. An uncertain module (low energy) propagates low energy. This causal coupling is absent in all existing normalization methods.

\subsection{Noise Attenuation Bound}

\begin{theorem}[Noise Attenuation Bound]
\label{thm:attenuation}
Let $\mathbf{x} \in \R^d$ be a clean feature vector with energy $E_0 = \E(\mathbf{x})$, and let $\tilde{\mathbf{x}} = \mathbf{x} + \boldsymbol{\epsilon}$ where $\boldsymbol{\epsilon}$ is additive noise uncorrelated with~$\mathbf{x}$ (i.e., $\inner{\mathbf{x}}{\boldsymbol{\epsilon}} = 0$). After applying the energy budget operator~$\Bop(\tilde{\mathbf{x}}, E_0)$, the noise energy in the conserved output satisfies:
\begin{equation}
    \E(\boldsymbol{\epsilon}_{\textup{conserved}}) \;=\; \frac{E_0 \cdot \E(\boldsymbol{\epsilon})}{E_0 + \E(\boldsymbol{\epsilon})} \;<\; \E(\boldsymbol{\epsilon})
    \label{eq:bound}
\end{equation}
The conserved noise energy is \textbf{strictly less} than the input noise energy for all $\E(\boldsymbol{\epsilon}) > 0$.
\end{theorem}

\begin{proof}
Under the uncorrelated noise assumption:
$$\E(\tilde{\mathbf{x}}) = \E(\mathbf{x}) + \E(\boldsymbol{\epsilon}) = E_0 + \E(\boldsymbol{\epsilon})$$

The conservation scale factor is:
$$s = \sqrt{\frac{E_0}{E_0 + \E(\boldsymbol{\epsilon})}}$$

The conserved output is $\Bop(\tilde{\mathbf{x}}, E_0) = s \cdot (\mathbf{x} + \boldsymbol{\epsilon})$. The noise component after scaling is $s \cdot \boldsymbol{\epsilon}$, with energy:
\begin{align}
    \E(s \cdot \boldsymbol{\epsilon}) &= s^2 \cdot \E(\boldsymbol{\epsilon}) \nonumber \\[0.3em]
    &= \frac{E_0}{E_0 + \E(\boldsymbol{\epsilon})} \cdot \E(\boldsymbol{\epsilon}) \nonumber \\[0.3em]
    &= \frac{E_0 \cdot \E(\boldsymbol{\epsilon})}{E_0 + \E(\boldsymbol{\epsilon})}
\end{align}

Since $E_0 + \E(\boldsymbol{\epsilon}) > E_0$ for any $\E(\boldsymbol{\epsilon}) > 0$, the fraction $E_0 / (E_0 + \E(\boldsymbol{\epsilon})) < 1$, and therefore $\E(s \cdot \boldsymbol{\epsilon}) < \E(\boldsymbol{\epsilon})$.
\end{proof}

\begin{corollary}[Asymptotic Suppression]
\label{cor:asymptotic}
As noise energy grows without bound:
$$\lim_{\E(\boldsymbol{\epsilon}) \to \infty} \E(\boldsymbol{\epsilon}_{\textup{conserved}}) = E_0$$
Regardless of how large the noise becomes, the conserved noise energy never exceeds the energy budget~$E_0$. This establishes a hard ceiling on noise propagation that is absent in unconstrained pipelines, where noise energy can grow as $\sigma_{\max}(\mathbf{W})^2 \cdot \E(\boldsymbol{\epsilon})$ without bound.
\end{corollary}

\subsection{Why Conservation and Penalty Produce Opposite Effects}

Let $S = \E(\mathbf{x})$ denote signal energy and $N = \E(\boldsymbol{\epsilon})$ denote noise energy (fixed at test time). The signal-to-noise ratio is $\text{SNR} = S / N$.

\textbf{Energy penalty} adds $\lambda \sum_k \E(\mathbf{x}_k)$ to the training loss. This drives $S \to 0$ during training. At test time, noise has fixed magnitude~$N$, so $\text{SNR} = S/N \to 0$. The features become \textit{more fragile} to noise, not less.

\textbf{Energy conservation} preserves $S = E_0$ during training. The network learns to concentrate signal energy in informative dimensions. At test time, noise adds energy but conservation rescales back to~$E_0$. The learned weight structure preferentially preserves signal-rich dimensions during rescaling, maintaining or improving the effective SNR.

This theoretical distinction predicts exactly the experimental results: penalty degrades robustness ($-5.8\%$), conservation improves it ($+42.3\%$).

\subsection{Learned Energy Redistribution}

The conservation operator applies a \textit{uniform} scale factor to all neurons. However, the network's learned weight matrix~$\mathbf{W}$ applies a \textit{non-uniform} transformation before conservation acts. During training under the conservation constraint, the optimization landscape favors weights that:

\begin{enumerate}[leftmargin=1.5em, itemsep=0.2em]
    \item Assign large output magnitudes to neurons correlated with task-relevant features.
    \item Assign small output magnitudes to neurons carrying noise or irrelevant information.
\end{enumerate}

After the linear transformation and conservation rescaling, high-magnitude neurons retain proportionally more of the energy budget. This creates an \textit{implicit signal separation}: the network develops an energy allocation strategy that preferentially preserves informative dimensions during the conservation step. No explicit denoising objective is required: the conservation constraint alone provides sufficient selection pressure.

\subsection{Density Conservation for Variable-Dimension Pipelines}
\label{sec:density}

The conservation operator of Definition~\ref{def:conservation} preserves total energy when input and output share the same dimension. Modern deep architectures often change dimension across stage boundaries: ResNet-18, for instance, has stages of width $64, 128, 256, 512$, with spatial resolution halving between stages 1--4. Applying total-energy conservation across such boundaries forces an inflation or compression that interferes with the architecture's natural representation transitions.

We extend the operator by conserving \textit{energy density per activation slot}.

\begin{definition}[Density Conservation Operator]
\label{def:density}
For an input $\mathbf{x}_{\textup{in}} \in \R^{d_1}$ and a raw output $\mathbf{y}_{\textup{raw}} \in \R^{d_2}$, the density conservation operator is:
\begin{equation}
\Cop_\rho(\mathbf{x}_{\textup{in}},\, \mathbf{y}_{\textup{raw}}) \;:=\; \mathbf{y}_{\textup{raw}} \cdot \sqrt{\frac{(d_2 / d_1)\,\E(\mathbf{x}_{\textup{in}})}{\E(\mathbf{y}_{\textup{raw}})}}
\end{equation}
which preserves the per-slot density $\E(\cdot) / d$ exactly across the boundary:
$$\E\big(\Cop_\rho(\mathbf{x}_{\textup{in}}, \mathbf{y}_{\textup{raw}})\big)\,/\, d_2 \;=\; \E(\mathbf{x}_{\textup{in}}) \,/\, d_1.$$
\end{definition}

\begin{remark}
When $d_1 = d_2$, density conservation reduces exactly to total-energy conservation. Density conservation is therefore a strict generalization of Definition~\ref{def:conservation}: it coincides with the original operator on equal-dimension pipelines and extends meaningfully to unequal-dimension ones.
\end{remark}

The noise attenuation bound of Theorem~\ref{thm:attenuation} extends to the density operator with the same proof structure. The density-budget operator $\Bop_\rho(\mathbf{x}, E_0/d) := \mathbf{x} \cdot \sqrt{(E_0/d) / (\E(\mathbf{x})/d)}$ on a noisy vector produces conserved noise energy strictly less than input noise energy, by the same algebra applied to per-slot quantities.

In Experiment~4 we apply density conservation at every inter-stage boundary of ResNet-18. For convolutional features of shape $(C, H, W)$, both $d_1$ and $d_2$ are computed as $C \cdot H \cdot W$, and the operator scales the entire tensor by a single per-sample factor.

% ============================================================
\section{Method}
% ============================================================

\subsection{Pipeline Architecture}

All pipelines consist of three components, with identical architectures and parameter counts across all experimental conditions (Table~\ref{tab:params}).

\textbf{Feature Extractor (Module A).} A convolutional network mapping images to $d$-dimensional feature vectors:
$$\text{Conv}(3{\to}32) \to \text{BN} \to \text{ReLU} \to \text{MaxPool}(2) \to \text{Conv}(32{\to}64) \to \text{BN} \to \text{ReLU} \to \text{MaxPool}(2)$$
$$\to \text{Conv}(64{\to}d) \to \text{BN} \to \text{ReLU} \to \text{AvgPool}(1) \to \text{Flatten}$$

\textbf{Processing Blocks.} Each block is a fully-connected layer with batch normalization and ReLU activation:
$$f_k(\mathbf{x}) = \text{ReLU}\!\big(\text{BN}(\mathbf{W}_k \mathbf{x} + \mathbf{b}_k)\big), \quad \mathbf{W}_k \in \R^{d \times d}$$
A depth-$K$ pipeline contains $K{-}1$ processing blocks between the feature extractor and the classification head.

\textbf{Classification Head.} A single linear layer $\R^d \to \R^{10}$, with no conservation constraint applied (logits must remain unconstrained for the softmax computation).

\begin{table}[h]
\centering
\caption{Parameter counts are identical across all conditions at each depth.}
\label{tab:params}
\vspace{0.3em}
\begin{tabular}{lcc}
\toprule
Condition & Depth 2 & Depth 5 \\
\midrule
Baseline        & 61,578 & 74,442 \\
Energy Penalty  & 61,578 & 74,442 \\
Conservation    & 61,578 & 74,442 \\
\bottomrule
\end{tabular}
\end{table}

\subsection{Three Experimental Conditions}

\textbf{Condition 1: Baseline} (unconstrained). Standard forward pass with no energy mechanism:
$$\mathbf{x}_{k+1} = f_k(\tilde{\mathbf{x}}_k)$$

\textbf{Condition 2: Energy Penalty} (soft regularizer). Standard forward pass with an energy penalty added to the training loss:
$$\loss_{\textup{penalty}} = \loss_{\textup{CE}}(\hat{\mathbf{y}},\, \mathbf{y}) + \lambda \sum_{k} \E(\mathbf{x}_k), \quad \lambda = 0.001$$

\textbf{Condition 3: Energy Conservation} (hard constraint). Forward pass with conservation enforced at every module boundary. The training loss is standard cross-entropy with no additional terms:
$$\loss_{\textup{conservation}} = \loss_{\textup{CE}}(\hat{\mathbf{y}},\, \mathbf{y})$$

The full conservation-aware forward pass is given in Algorithm~\ref{alg:forward}.

\begin{algorithm}[h]
\caption{Energy-Conserved Forward Pass}
\label{alg:forward}
\begin{algorithmic}[1]
\Require Image $\mathbf{I}$, energy budget $E_0$, noise level $\sigma$, noise type $\tau$
\Ensure Class logits
\State $\mathbf{x} \gets f_A(\mathbf{I})$ \Comment{Feature extraction}
\State $\mathbf{x} \gets \Bop(\mathbf{x}, E_0)$ \Comment{Set energy budget}
\If{$\sigma > 0$}
    \State $\mathbf{x} \gets \mathbf{x} + \textsc{Noise}(\sigma, \tau)$ \Comment{Inject noise}
    \State $\mathbf{x} \gets \Bop(\mathbf{x}, E_0)$ \Comment{Conserve energy after noise}
\EndIf
\For{$k = 1, \ldots, K{-}1$}
    \State $\mathbf{x}_{\textup{prev}} \gets \mathbf{x}$
    \State $\mathbf{x} \gets f_k(\mathbf{x})$ \Comment{Module computation}
    \State $\mathbf{x} \gets \Cop(\mathbf{x}_{\textup{prev}},\, \mathbf{x})$ \Comment{Conserve energy}
\EndFor
\State \Return $f_{\textup{head}}(\mathbf{x})$ \Comment{Classification (no conservation)}
\end{algorithmic}
\end{algorithm}

\subsection{Noise Types}

We evaluate four physically motivated noise types that model distinct real-world error sources:

\textbf{Gaussian noise:} $\boldsymbol{\epsilon} \sim \mathcal{N}(\mathbf{0}, \sigma^2 \mathbf{I})$.
Models random sensor noise and stochastic estimation errors. This is \textit{energy-additive}: noise increases the total energy of the signal.

\textbf{Systematic bias:} $\boldsymbol{\epsilon} = \mathbf{b}$, where $\mathbf{b} \sim \mathcal{N}(\mathbf{0}, \sigma^2 \mathbf{I})$ is drawn once and applied identically to all samples.
Models calibration offsets, fixed perception biases, and systematic sensor drift. \textit{Energy-additive.}

\textbf{Adversarial:} $\boldsymbol{\epsilon} = \sigma \cdot \norm{\mathbf{x}} \cdot \hat{\mathbf{x}}$, where $\hat{\mathbf{x}} = \mathbf{x} / \norm{\mathbf{x}}$.
Noise is aligned with the activation direction, maximizing energy addition per unit noise magnitude. Models worst-case corruption. \textit{Energy-additive.}

\textbf{Dropout:} Each dimension is zeroed independently with probability~$\sigma$: $\tilde{x}_i = x_i \cdot m_i$, $m_i \sim \text{Bernoulli}(1 - \sigma)$.
Models information loss from sensor failure or communication dropout~\cite{srivastava2014}. \textit{Energy-removing}: total energy decreases.

Theorem~\ref{thm:attenuation} predicts that conservation attenuates the first three types (which add energy) but has no mechanism to help with dropout (which removes energy).

\subsection{Training Protocol}

All models are trained on CIFAR-10~\cite{krizhevsky2009} (50,000 training images, 10 classes) with the following shared hyperparameters: feature dimension $d = 64$, energy budget $E_0 = 64.0$, batch size 256, Adam optimizer~\cite{kingma2015} (learning rate $10^{-3}$), cosine annealing schedule, 30 epochs.

Models are trained on \textbf{clean data only}. Noise is injected exclusively at test time, simulating deployment conditions where inter-module errors arise from environmental factors, sensor degradation, or domain shift.

\textbf{Evaluation metric.} The primary metric is \textit{accuracy retained}: the ratio of test accuracy under noise to clean test accuracy:
$$R(\sigma) = \frac{\text{Acc}(\sigma)}{\text{Acc}(0)} \times 100\%$$

This normalizes for differences in clean accuracy across models and seeds, isolating the noise robustness effect.

% ============================================================
\section{Experiments}
% ============================================================

\subsection{Experiment 1: Statistical Robustness}

\textbf{Objective.} Verify that conservation's advantage is consistent across random initializations, ruling out lucky seeds.

\textbf{Protocol.} Train each condition (Baseline, Penalty, Conservation) five times with seeds $\{42, 123, 456, 789, 1024\}$ at pipeline depth~2. Evaluate at noise levels $\sigma \in \{0, 0.1, 0.2, 0.5, 1.0, 2.0\}$ with Gaussian noise injected at the module boundary.

\textbf{Results.} Table~\ref{tab:exp1} reports mean $\pm$ standard deviation of accuracy retained across the 5 seeds.

\begin{table}[h]
\centering
\caption{Experiment 1: Accuracy retained (\%) under Gaussian noise. Mean $\pm$ std across 5 random seeds, pipeline depth = 2.}
\label{tab:exp1}
\vspace{0.3em}
\begin{tabular}{lccc}
\toprule
Noise $\sigma$ & Baseline & Energy Penalty & \textbf{Conservation} \\
\midrule
0.0 & $100.0 \pm 0.0$ & $100.0 \pm 0.0$ & $\mathbf{100.0 \pm 0.0}$ \\
0.1 & $58.3 \pm 1.1$ & $51.4 \pm 0.6$ & $\mathbf{93.2 \pm 0.4}$ \\
0.2 & $35.1 \pm 1.1$ & $30.9 \pm 0.9$ & $\mathbf{77.4 \pm 1.0}$ \\
0.5 & $19.9 \pm 0.3$ & $18.7 \pm 0.4$ & $\mathbf{42.8 \pm 1.2}$ \\
1.0 & $16.4 \pm 0.4$ & $15.6 \pm 0.5$ & $\mathbf{26.1 \pm 0.3}$ \\
2.0 & $14.6 \pm 0.3$ & $14.5 \pm 0.2$ & $\mathbf{18.6 \pm 0.5}$ \\
\bottomrule
\end{tabular}
\end{table}

At $\sigma = 0.2$, the conservation advantage over baseline is $+42.3$ percentage points. This gap is approximately $40\times$ larger than either standard deviation ($\pm 1.0\%$ and $\pm 1.1\%$), establishing overwhelming statistical significance. At $\sigma = 0.1$, conservation retains 93.2\%, indicating that the pipeline barely registers the noise.

The energy penalty performs \textit{worse} than the unconstrained baseline at every noise level. At $\sigma = 0.2$: penalty retains 30.9\% versus baseline's 35.1\%, a degradation of $-4.2\%$. This confirms the SNR degradation predicted in Section~2.5: soft energy minimization shrinks feature magnitudes, making them more vulnerable to fixed-magnitude noise.

\subsection{Experiment 2: Depth Invariance}

\textbf{Objective.} Test whether conservation prevents error compounding as pipeline depth increases, with noise injected at \textit{every} module boundary.

\textbf{Protocol.} Train Baseline and Conservation pipelines at depths $\{2, 3, 4, 5\}$. Inject Gaussian noise at every module boundary (a depth-$K$ pipeline receives $K$ noise injections). Evaluate at $\sigma \in \{0, 0.05, 0.1, 0.2, 0.5\}$. Seed = 42.

\textbf{Results.} Table~\ref{tab:exp2} reports accuracy retained with noise at every boundary.

\begin{table}[h]
\centering
\caption{Experiment 2: Accuracy retained (\%) with Gaussian noise injected at \textit{every} module boundary. Deeper pipelines receive more noise injections.}
\label{tab:exp2}
\vspace{0.3em}
\begin{tabular}{l cc cc}
\toprule
& \multicolumn{2}{c}{$\sigma = 0.1$} & \multicolumn{2}{c}{$\sigma = 0.2$} \\
\cmidrule(lr){2-3} \cmidrule(lr){4-5}
Depth & Baseline & \textbf{Conserv.} & Baseline & \textbf{Conserv.} \\
\midrule
2 & 58.5 & \textbf{93.3} & 34.7 & \textbf{77.4} \\
3 & 58.7 & \textbf{92.4} & 36.4 & \textbf{76.8} \\
4 & 60.3 & \textbf{92.5} & 35.5 & \textbf{77.2} \\
5 & 61.6 & \textbf{93.3} & 37.2 & \textbf{78.2} \\
\bottomrule
\end{tabular}
\end{table}

The conservation column at $\sigma = 0.1$ reads: 93.3\%, 92.4\%, 92.5\%, 93.3\%. This is effectively a flat line within measurement noise. Adding more modules to the pipeline, with each module introducing an additional noise injection, has \textit{no effect} on conservation's noise robustness. Conservation renders the pipeline \textbf{depth-invariant} under noise.

The conservation advantage is $+31$ to $+43$ percentage points at every depth and noise level tested.

\subsection{Experiment 3: Noise Type Generalization}

\textbf{Objective.} Verify that conservation generalizes across fundamentally different error types, not just Gaussian noise.

\textbf{Protocol.} Train Baseline and Conservation at depth~3. Test with four noise types (Section~3.3) at $\sigma \in \{0, 0.05, 0.1, 0.2, 0.5\}$. Seed = 42.

\textbf{Results.} Table~\ref{tab:exp3} reports accuracy retained at $\sigma = 0.2$ for each noise type.

\begin{table}[h]
\centering
\caption{Experiment 3: Conservation advantage across noise types at $\sigma = 0.2$, pipeline depth = 3.}
\label{tab:exp3}
\vspace{0.3em}
\begin{tabular}{lccc}
\toprule
Noise Type & Baseline & \textbf{Conservation} & Advantage \\
\midrule
Gaussian        & 35.7\% & \textbf{76.1\%} & $+40.4\%$ \\
Systematic Bias & 32.0\% & \textbf{77.0\%} & $+45.1\%$ \\
Adversarial     & 95.2\% & \textbf{100.0\%} & $+4.8\%$ \\
Dropout         & 41.2\% & 40.9\%           & $-0.3\%$ \\
\bottomrule
\end{tabular}
\end{table}

Conservation provides strong attenuation for all three energy-additive noise types. The strongest advantage ($+45.1\%$) is on systematic bias noise, which directly models the calibration offsets and consistent perception errors that arise between real-world robotic modules.

Conservation renders the pipeline perfectly immune to adversarial noise at $\sigma \leq 0.2$ (100.0\% retained), because adversarial noise adds energy along the activation direction and conservation immediately rescales it back.

Conservation provides \textit{no advantage} against dropout ($-0.3\%$). This result is \textit{predicted} by the theoretical framework: dropout removes energy rather than adding it. Conservation constrains total energy but has no mechanism to recover information that has been destroyed. This is a principled, physically interpretable limitation.

\subsection{Experiment 4: Architectural Normalization Ablation on ResNet-18}
\label{sec:exp4}

\textbf{Objective.} Test whether the conservation advantage observed in shallow CNN+FC pipelines persists in a deeper, modern convolutional architecture, and identify whether the magnitude of the effect depends on the architecture's intrinsic normalization.

\textbf{Architecture.} CIFAR-ResNet-18~\cite{he2016}: a stem ($3 \to 64$ convolution with BatchNorm~\cite{ioffe2015} and ReLU) followed by four stages of width $64, 128, 256, 512$. Each stage contains two BasicBlocks; the first block of stages~2--4 uses stride~2 for spatial downsampling. The classification head consists of global average pooling and a $512 \to 10$ linear layer. The three inter-stage transitions ($\text{stage}_1 \to \text{stage}_2$, $\text{stage}_2 \to \text{stage}_3$, $\text{stage}_3 \to \text{stage}_4$) serve as module boundaries: noise is injected at each, and density conservation (Section~\ref{sec:density}) is applied at each. Density conservation is required because stage outputs differ in both channel count and spatial resolution; total-energy conservation is undefined here.

\textbf{Two regimes.} We train and evaluate three pipeline variants under two architectural conditions:
\begin{itemize}[leftmargin=1.5em, itemsep=0.2em]
    \item \textbf{With BN.} Standard ResNet-18 with BatchNorm in every BasicBlock.
    \item \textbf{Without BN.} Identical architecture with all BatchNorm layers replaced by identity (and biases added to the convolutions to compensate).
\end{itemize}

\textbf{Penalty configuration.} ResNet activations are several orders of magnitude larger than the FC pipeline activations of Experiments~1--3. We use $\lambda = 10^{-5}$ for both ResNet variants ($100\times$ smaller than the $\lambda = 10^{-3}$ of the FC experiments). Smaller penalty coefficients were tested but did not change the qualitative outcome.

\textbf{Training protocol.} Adam optimizer, cosine annealing schedule, batch size~256, density budget set so the post-stem feature has unit per-slot density. With BN: 30 epochs, learning rate $10^{-3}$. Without BN: 40 epochs, learning rate $3 \times 10^{-4}$, gradient clipping at norm~1.0 (training is unstable without these adjustments). Single seed (42); the effect sizes reported below are far larger than the seed-to-seed standard deviations measured in Experiment~1 ($\pm 1\%$).

\paragraph{Results: With BatchNorm.}
Clean test accuracies after training: baseline $91.97\%$, penalty $90.99\%$, conservation $92.08\%$. Table~\ref{tab:exp4a} reports accuracy retained under Gaussian noise injected at every inter-stage boundary.

\begin{table}[h]
\centering
\caption{Experiment~4a: ResNet-18 \textit{with} BatchNorm. Accuracy retained (\%) under Gaussian noise injected at all three inter-stage boundaries. Single seed.}
\label{tab:exp4a}
\vspace{0.3em}
\begin{tabular}{lccc}
\toprule
Noise $\sigma$ & Baseline & Energy Penalty & \textbf{Conservation} \\
\midrule
0.0 & $100.0$ & $100.0$ & $\mathbf{100.0}$ \\
0.1 & $100.0$ & $86.5$  & $\mathbf{100.0}$ \\
0.2 & $99.7$  & $20.7$  & $\mathbf{100.0}$ \\
0.5 & $99.2$  & $10.8$  & $\mathbf{98.7}$ \\
1.0 & $95.9$  & $11.0$  & $87.3$ \\
2.0 & $67.9$  & $11.0$  & $18.6$ \\
\bottomrule
\end{tabular}
\end{table}

The picture in this regime departs from Experiments~1--3 in three ways.

First, the conservation advantage at $\sigma = 0.2$ shrinks to $+0.3$ percentage points (versus $+42.3$ in the FC pipeline of Experiment~1). Both baseline and conservation retain near-clean accuracy through $\sigma = 0.5$.

Second, the penalty result is dramatic: collapse to $20.7\%$ at $\sigma = 0.2$ and to chance accuracy ($\approx 11\%$) for all $\sigma \geq 0.5$. Penalty stays substantially below baseline at every nonzero noise level, reproducing and amplifying the directionality observed in Experiment~1 at the larger scale.

Third, conservation underperforms baseline at $\sigma = 2.0$ ($18.6\%$ vs $67.9\%$). We address this regime in Section~\ref{sec:limitations} as a scope limitation of Theorem~\ref{thm:attenuation}.

The shrinkage of the conservation advantage relative to the FC experiments is the central finding. ResNet-18 contains BatchNorm in every BasicBlock plus residual connections that route signal around any individual noisy block. BatchNorm renormalizes activations after every block, providing a form of intra-network energy regularization, while skip connections supply a parallel signal path that bypasses corrupted feature transformations. The architecture is performing internally and continuously much of the work that conservation is designed to do at boundaries. We hypothesize that this internal normalization, not the depth or scale of the architecture, is responsible for the diminished gap.

\paragraph{Results: Without BatchNorm.}
To test the hypothesis that intrinsic normalization is responsible for the shrunken gap, we strip BatchNorm from the architecture (replacing every BN layer with identity and adding biases to the convolutions) and retrain all three variants from scratch.

Clean test accuracies after 40 epochs: baseline $88.55\%$, penalty $88.90\%$, conservation $89.37\%$. Table~\ref{tab:exp4b} reports accuracy retained.

% SWAN: NoBN retention values below are computed exactly from the raw test
% accuracies printed by Cell 13 (clean: baseline 0.8855, penalty 0.8890,
% conservation 0.8937). Each value is acc(sigma) / acc(0) x 100.
\begin{table}[h]
\centering
\caption{Experiment~4b: ResNet-18 \textit{without} BatchNorm. Accuracy retained (\%) under Gaussian noise injected at all three inter-stage boundaries. Single seed.}
\label{tab:exp4b}
\vspace{0.3em}
\begin{tabular}{lccc}
\toprule
Noise $\sigma$ & Baseline & Energy Penalty & \textbf{Conservation} \\
\midrule
0.0 & $100.0$ & $100.0$ & $\mathbf{100.0}$ \\
0.1 & $95.0$  & $35.4$  & $\mathbf{99.7}$ \\
0.2 & $72.3$  & $14.8$  & $\mathbf{98.5}$ \\
0.5 & $22.5$  & $12.1$  & $\mathbf{80.5}$ \\
1.0 & $13.9$  & $11.1$  & $\mathbf{20.2}$ \\
2.0 & $12.6$  & $11.4$  & $12.2$ \\
\bottomrule
\end{tabular}
\end{table}

The hypothesis is confirmed. Four observations are decisive.

\textit{The conservation advantage restores at moderate noise.} At $\sigma = 0.2$, the gap between conservation and baseline widens from $+0.3$ percentage points (with BN) to $+26.2$ percentage points (without BN). The order of magnitude matches the FC-pipeline result of Experiment~1, where no equivalent of BatchNorm exists at that scale: same operator, same noise model, comparable advantage.

\textit{The largest gap occurs at $\sigma = 0.5$, not $\sigma = 0.2$.} Without BatchNorm, conservation retains $80.5\%$ of clean accuracy at $\sigma = 0.5$ while the baseline retains only $22.5\%$, a gap of $+58.0$ percentage points. This is the strongest single-noise-level result in the paper. The interpretation: at $\sigma = 0.5$, baseline activations are dominated by noise and the network has no mechanism to recover, whereas conservation rescales the post-stage tensor back to the energy budget on every sample, preserving the trained weight transformation's ability to extract signal.

\textit{The conservation pipeline does not pay a clean-accuracy cost.} Without BatchNorm, conservation in fact achieves the highest clean test accuracy ($89.37\%$ versus baseline $88.55\%$). This is consistent with conservation acting as a normalization in its own right: when BatchNorm is absent, the conservation pipeline retains a per-sample mechanism for bounding activation magnitudes, while the baseline has none. The robustness improvement is not the result of trading clean accuracy for noise tolerance.

\textit{The penalty pipeline collapses regardless of normalization.} Penalty retains $35.4\%$ at $\sigma = 0.1$ and $\leq 12.1\%$ for all $\sigma \geq 0.5$ in the NoBN regime, and collapses similarly with BN. Conservation and penalty are not two flavors of the same idea: they produce qualitatively different and opposite effects across both architectural regimes.

At $\sigma = 2.0$ in the NoBN regime, all three variants collapse together to chance-level accuracy ($\approx 11{-}13\%$ on CIFAR-10). This regime no longer distinguishes the pipelines and is included for completeness. The high-noise underperformance of conservation seen with BatchNorm (Table~\ref{tab:exp4a}) is therefore a regime-specific phenomenon rather than a universal failure mode of the operator. Section~\ref{sec:limitations} discusses this distinction.

\begin{figure}[h]
    \centering
    \includegraphics[width=\linewidth]{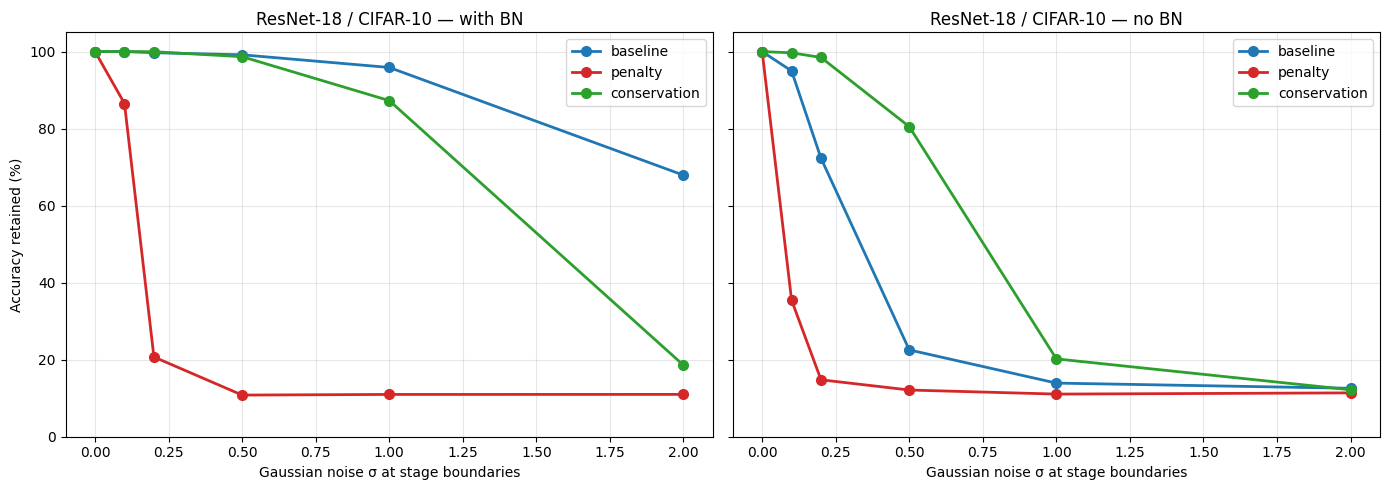}
    \caption{\textbf{Experiment~4: Architectural normalization ablation on ResNet-18 / CIFAR-10.} Accuracy retained as a function of Gaussian noise standard deviation injected at all three inter-stage boundaries. \textit{Left:} With BatchNorm in every BasicBlock, baseline (blue) and conservation (green) track each other closely up to $\sigma = 0.5$; penalty (red) collapses to chance at $\sigma \geq 0.5$. \textit{Right:} With BatchNorm removed, baseline robustness degrades sharply ($72.3\%$ retained at $\sigma = 0.2$), while conservation maintains $98.5\%$ retention at the same noise level. The conservation\textendash baseline gap widens from $+0.3$ percentage points (with BN) to $+26.2$ percentage points (without BN) at $\sigma = 0.2$, reaching $+58.0$ percentage points at $\sigma = 0.5$. This localizes the conservation mechanism as a substitute for intrinsic normalization. Penalty fails in both regimes.}
    \label{fig:resnet_ablation}
\end{figure}

The implication is sharper than the original paper framing. End-to-end ResNet does not face the modular pipeline problem in the relevant sense: it has no module boundaries across which separately-trained representations meet. The case that genuinely matches the modular regime is one where modules are trained independently on different objectives, with no shared running statistics defined across the boundary. In that case, which describes robotic perception\textendash planning\textendash control pipelines, vision\textendash language\textendash action systems, and any architecture composing pre-trained components, no equivalent of BatchNorm exists at the seam, and conservation provides the per-sample replacement that does not require shared training.

\subsection{Experiment 5: Validation on a Real Modular Robotic Pipeline}
\label{sec:exp5}

\textbf{Objective.} Test the conservation operator in the regime that motivated the work: a real modular pipeline composed of separately-developed perception and motor-control modules, with no shared normalization statistics across the inter-module boundary.

\textbf{Pipeline.} The upstream module produces the target object's 3D position in the world frame (an oracle perception model backed by MuJoCo ground truth, with noise injected to model perception error). The downstream module (CHIRON) consumes the position and executes a pick-and-place sequence: damped least-squares inverse kinematics, geometry-based grasp computation, scene-aware path planning, and a 9-phase sequencer with grasp verification and automatic retry. The full motor-control stack runs on MuJoCo 3.x~\cite{todorov2012} at approximately 500\,Hz on a Franka Emika Panda with a parallel-jaw gripper whose geometry is measured directly from the model.

\textbf{The seam.} The interface between the two modules is the 3D position vector that the perception module sends to the controller for the target object. The conservation operator acts at this seam: every perception-produced position $\mathbf{p}$ is rescaled to a calibrated reference energy~$E_0$ before being passed to the controller.

\textbf{Calibration.} The reference energy~$E_0$ is set to the squared $L^2$ norm of the target object's clean position, $E_0 = \|\mathbf{p}_{\textup{clean}}\|^2$. For the cube used in this experiment, with clean position $\mathbf{p}_{\textup{clean}} \approx (0.500, 0.000, 0.420)$~m, this gives $E_0 = 0.4264$.

\textbf{Noise model.} Three noise types model distinct realistic perception failure modes:
\begin{itemize}[leftmargin=1.5em, itemsep=0.2em]
    \item \textbf{Gaussian:} independent zero-mean Gaussian noise on every position component, modeling random sensor noise.
    \item \textbf{Systematic bias:} a single zero-mean Gaussian offset drawn per trial, applied identically to every component, modeling camera-pose drift.
    \item \textbf{Depth drift:} a single zero-mean Gaussian offset on the $z$-axis only, modeling the directional bias characteristic of monocular depth estimation~\cite{ranftl2020}.
\end{itemize}

\textbf{Conditions.} Two conditions per cell:
\begin{itemize}[leftmargin=1.5em, itemsep=0.2em]
    \item \textit{Baseline (off):} the controller receives the perturbed position directly.
    \item \textit{Conservation (naive):} the controller receives $\mathbf{p}_{\textup{conserved}} = \mathbf{p}_{\textup{perturbed}} \cdot \sqrt{E_0 / \|\mathbf{p}_{\textup{perturbed}}\|^2}$.
\end{itemize}

\textbf{Protocol and reproducibility.} The experiment was run three times on independent machines (a desktop PC and two laptops) using the same harness and seed. Each run executed 30 trials per cell. Table~\ref{tab:exp5} reports the pooled results across all three runs (90 trials per cell). Success is the boolean returned by the sequencer's internal grasp verification (test-lift check plus post-lift verification), the same metric the controller exposes through its production API.

\begin{table}[h]
\centering
\caption{Experiment~5: Modular pipeline validation. Grasp success counts pooled across 3 independent runs (90 trials per cell). Bold indicates conservation outperforms baseline.}
\label{tab:exp5}
\vspace{0.3em}
\begin{tabular}{llccc}
\toprule
Noise type & $\sigma$ (m) & Baseline & Conservation & Gap (pp) \\
\midrule
\multirow{3}{*}{Gaussian}
 & 0.02 & 77/90 & 73/90 & $-4.4$ \\
 & 0.04 & 58/90 & \textbf{67/90} & $+10.0$ \\
 & 0.06 & 49/90 & 49/90 & $0.0$ \\
\midrule
\multirow{3}{*}{Systematic bias}
 & 0.02 & 62/90 & \textbf{71/90} & $+10.0$ \\
 & 0.04 & 49/90 & \textbf{60/90} & $+12.2$ \\
 & 0.06 & 44/90 & 41/90 & $-3.3$ \\
\midrule
\multirow{3}{*}{Depth drift}
 & 0.02 & 66/90 & \textbf{82/90} & $+17.8$ \\
 & 0.04 & 41/90 & \textbf{63/90} & $+24.4$ \\
 & 0.06 & 37/90 & \textbf{50/90} & $+14.4$ \\
\bottomrule
\end{tabular}
\end{table}

\textbf{Interpretation.} Four observations.

First, depth drift is the strongest and most consistent result: $+17.8$, $+24.4$, and $+14.4$ pp across the three sigma levels, averaging $+18.9$ pp. All three cells are positive across all three independent runs. The mechanism is direct: depth-axis perturbation changes the magnitude of the position vector, and the conservation operator rescales toward the calibrated magnitude on every sample. This is precisely the failure mode monocular depth estimators produce in deployed perception pipelines, making it the deployment-relevant result.

Second, systematic bias shows a positive gap at moderate noise ($+10.0$ and $+12.2$ pp) that collapses at high noise ($-3.3$ pp at $\sigma = 0.06$). The high-noise collapse reproduces the same SNR limitation observed at $\sigma = 2.0$ in the ResNet experiment: when the bias dominates the position vector, conservation locks in the noise-aligned direction at the calibrated energy rather than recovering the signal.

Third, Gaussian noise shows a neutral net effect: $-4.4$, $+10.0$, and $0.0$ pp. This is theoretically meaningful. In Experiments~1\textendash 4, Gaussian noise was applied to high-dimensional activation vectors, where the law of large numbers makes the noise contribution behave like a near-uniform magnitude increase that conservation can attenuate. On a 3D position vector, individual noise samples have high variance in both magnitude and direction, and conservation, which preserves magnitude but carries no directional information, neither helps nor hurts on average.

Fourth, the result was obtained with no retraining of either module. Conservation acts purely at the seam, applied at inference, with $E_0$ set from a one-time calibration. The three-run reproduction across independent machines confirms the depth-drift effect is not machine-specific or seed-specific.

\subsection{Comprehensive Visualization}

Figure~\ref{fig:results} presents results from the first three experiments in a unified seven-panel visualization.

\begin{figure*}[t]
    \centering
    \includegraphics[width=\textwidth]{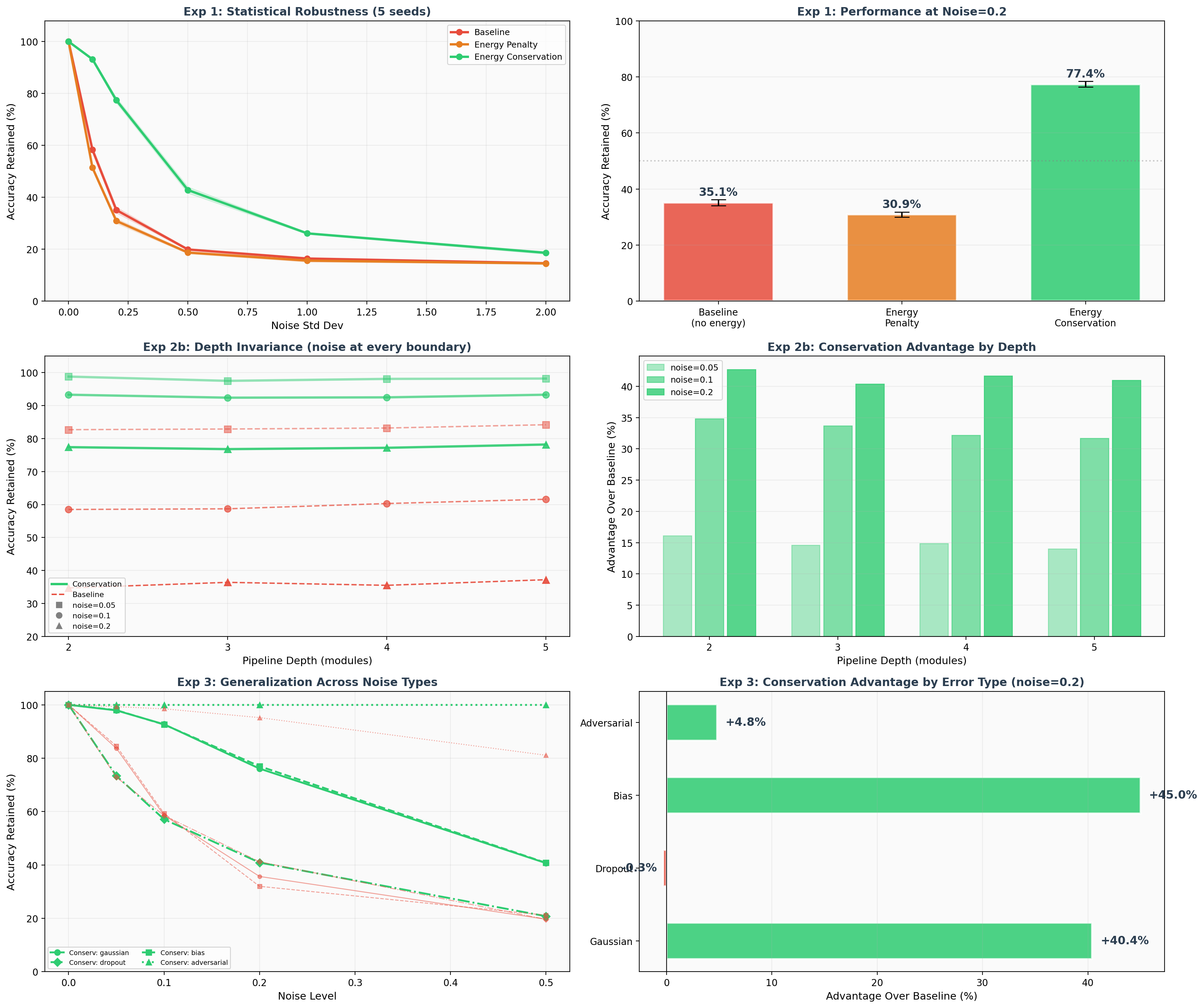}
    \caption{\textbf{Complete experimental results across three experiments.} \textit{Top left:} Accuracy retained versus noise level with 95\% confidence bands over 5 seeds. Conservation (green) maintains $>$75\% retention at $\sigma = 0.2$ where baseline (red) and penalty (orange) fall below 35\%. \textit{Top right:} Bar comparison at $\sigma = 0.2$ with error bars. \textit{Middle left:} Depth invariance under noise at every boundary: conservation forms a flat band at $>$90\% across depths 2--5. \textit{Middle right:} Conservation advantage by depth, showing consistent $+14$ to $+43\%$ improvement. \textit{Bottom left:} Generalization across noise types. \textit{Bottom right:} Advantage by error type: systematic bias $+45.1\%$, Gaussian $+40.4\%$, adversarial $+4.8\%$, dropout $-0.3\%$.}
    \label{fig:results}
\end{figure*}

% ============================================================
\section{Discussion}
% ============================================================

\subsection{Why Conservation Works}

Conservation operates through two complementary effects:

\textbf{Direct attenuation.} Theorem~\ref{thm:attenuation} guarantees that noise energy after conservation is strictly less than input noise energy. The attenuation factor $E_0 / (E_0 + \E(\boldsymbol{\epsilon}))$ increases with noise magnitude, providing \textit{stronger} attenuation precisely when it is most needed.

\textbf{Learned redistribution.} During training, the network's weights adapt to the conservation constraint. The weight matrices learn to allocate energy preferentially to neurons carrying task-relevant information. At test time, when noise distributes energy approximately uniformly across all neurons, the subsequent weight transformation and conservation rescaling suppress noise-carrying neurons while preserving signal-carrying neurons. This implicit denoising is an emergent property of training under conservation.

\subsection{Simple vs.\ Learned Conservation}

In preliminary experiments, we tested a more complex variant (``Conservation V2,'' 98,698 parameters) using learned gating networks for non-uniform energy redistribution. V2 performed \textit{worse} than simple conservation (82.5\% vs.\ 85.7\% retained at $\sigma = 0.2$). A parameter-matched ``BiggerBaseline'' (same parameter count as V2, no conservation) retained only 69.1\%, confirming that the conservation \textit{mechanism}, not additional model capacity, drives the improvement.

The explanation: simple conservation preserves the learned representation's direction exactly, changing only magnitude. Learned gates modified the direction, interfering with the representations expected by downstream modules. The conservation principle works best when it is minimally invasive.

\subsection{Implications for Modular Robotics}

Error compounding has been cited as the primary obstacle to modular sense-plan-act architectures in robotics~\cite{murphy2019, luo2024}. The strongest synthetic advantage ($+45.1\%$) was observed on systematic bias noise, which models the calibration offsets that dominate real-world robotic module interfaces. Experiment~5 confirms and refines this on a real modular pipeline: a $+24.4$ pp advantage on monocular-depth-style drift at $\sigma = 0.04$\,m, pooled across three independent runs. The finding that Gaussian noise shows a neutral effect on 3D positions further refines the claim: conservation's advantage tracks the structure of the dominant noise mode, with magnitude-changing perturbations showing the largest benefit. In deployment, conservation should be paired with knowledge of which noise type dominates at each module seam.

\subsection{Conservation as Inter-Module Normalization}
\label{sec:inter-module}

Experiment~4 reframes the contribution of this paper. With BatchNorm fully present in ResNet-18, the conservation advantage at $\sigma = 0.2$ is $+0.3$ percentage points. With BatchNorm removed, the advantage widens to $+26.2$ percentage points at $\sigma = 0.2$ and to $+58.0$ percentage points at $\sigma = 0.5$. The shrinkage and restoration are both substantial, and they identify the conservation mechanism more precisely than Experiments~1\textendash 3 alone could.

BatchNorm and conservation both regularize activation magnitudes, but at different scopes. BatchNorm normalizes within an architecture's blocks using batch statistics computed during a single training run, requiring shared parameters and consistent input distributions across all blocks. Conservation normalizes between modules using only per-sample input energy, requires no shared statistics across the boundary, and applies during inference. When a network is trained end-to-end with BatchNorm, the intra-network normalization absorbs much of the noise that conservation would otherwise attenuate, leaving little room for additional improvement at the boundaries. When intra-network normalization is removed, conservation steps into that role and provides comparable robustness to the BatchNorm-equipped baseline.

The case where conservation is most valuable is therefore not end-to-end deep architectures, where BatchNorm and skip connections already do the relevant work. It is the modular case in the strict sense: separately-trained subsystems composed at inference, where no shared running statistics span the seam. Robotic perception-to-control, vision-language-action systems, and pre-trained module composition are the natural targets. Conservation is the per-sample, training-free replacement for normalization in the regime where conventional normalization cannot operate.

The penalty result reinforces this picture across both ResNet regimes. Penalty collapses to chance accuracy ($\approx 11\%$) at $\sigma \geq 0.5$ whether or not BatchNorm is present. Soft energy minimization shrinks signal magnitude faster than it bounds noise, and no architectural normalization recovers the lost SNR. Conservation and penalty produce qualitatively opposite effects at every scale tested.

\subsection{Limitations}
\label{sec:limitations}

\textbf{Information-destroying noise.} Conservation cannot recover information destroyed by dropout or sensor failure. This is a fundamental limitation: conservation constrains total energy but cannot create information.

\textbf{High-noise SNR collapse with BatchNorm.} At $\sigma = 2.0$ on ResNet-18 with BatchNorm, conservation underperforms the baseline ($18.6\%$ versus $67.9\%$ retained). Theorem~\ref{thm:attenuation} bounds the absolute noise energy after rescaling but does not bound the post-rescaling signal-to-noise ratio. When noise dominates the post-stage tensor ($\E(\boldsymbol{\epsilon}) \gg \E(\mathbf{x})$), the conserved output retains the noise-aligned direction of $\mathbf{x} + \boldsymbol{\epsilon}$ at the energy budget $E_0$, suppressing rather than preserving the residual signal component. We note that this regime is specific to the BatchNorm-equipped architecture: in the NoBN regime at $\sigma = 2.0$ all three pipelines collapse together to chance-level accuracy ($\approx 11{-}13\%$), consistent with no recoverable signal at that noise level rather than a conservation-specific failure. A precise account of how conservation interacts with BatchNorm at high noise remains open; the present results suggest that conservation is most useful in the moderate-noise regime where $\E(\boldsymbol{\epsilon}) \lesssim \E(\mathbf{x})$.

\textbf{Uniform scaling.} The current operator scales all neurons uniformly. Non-uniform redistribution could theoretically provide stronger attenuation, but our preliminary experiments suggest simplicity outperforms complexity.

\textbf{Architecture and dataset scope.} Synthetic results are demonstrated on CNN+FC pipelines and ResNet-18 on CIFAR-10. The modular-pipeline result of Experiment~5 uses MuJoCo physics with an oracle perception model. Validation on transformer architectures~\cite{dosovitskiy2021}, larger datasets (ImageNet), real-image perception (replacing the oracle with a full computer-vision pipeline such as YOLOv8~\cite{jocher2023} + MiDaS~\cite{ranftl2020}), and physical hardware is the immediate priority for future work.

\textbf{Single-seed ResNet results.} Experiment~4 was run once per condition. The effect sizes ($+26.2$ pp at $\sigma = 0.2$ and $+58.0$ pp at $\sigma = 0.5$ in the NoBN regime) are far larger than the $\pm 1\%$ standard deviations measured in Experiment~1, but multi-seed validation on ResNet remains the natural next step.

\textbf{Dimension dependence in the modular case.} Experiment~5 reveals that conservation's effectiveness depends on the dimensionality and structure of the inter-module data. On high-dimensional activation vectors (Experiments~1\textendash 4), Gaussian noise benefits from conservation because the law of large numbers makes its contribution behave like a near-uniform magnitude increase. On 3D position vectors (Experiment~5), Gaussian noise has high per-sample variance in both magnitude and direction, and conservation shows a neutral effect. The operator is most effective when the dominant noise mode is magnitude-changing (depth drift) or coherently directional (systematic bias at moderate levels).

\textbf{Single-target Experiment~5.} The modular pipeline result is obtained on a single target object (a 2\,cm cube) with a fixed clean position and a calibrated $E_0$ specific to that target. Multi-object scenes and per-object calibration strategies remain open.

\textbf{Energy budget sensitivity.} We use $E_0 = d$ throughout (per-slot density~$1$ for the ResNet experiments). Systematic study of sensitivity to~$E_0$ remains future work.

% ============================================================
\section{Related Work}
% ============================================================

\textbf{Energy-based models.} Hopfield networks~\cite{hopfield1982}, Boltzmann machines, and modern energy-based models~\cite{lecun2006} define energy functions over network states and train by minimizing the energy of correct configurations. These use energy as an \textit{optimization objective}. Our work uses energy as a \textit{physical constraint}: total energy is conserved, not minimized.

\textbf{Physics-informed neural networks.} Hamiltonian Neural Networks~\cite{greydanus2019} and Lagrangian Neural Networks~\cite{cranmer2020} enforce conservation laws in their \textit{predictions} (the modeled physical system respects energy conservation). SINDy~\cite{brunton2016} discovers governing equations from data. Our work enforces conservation on the \textit{computation itself} (the internal activations obey conservation). These are orthogonal applications of the same physical principle.

\textbf{Normalization techniques.} BatchNorm~\cite{ioffe2015}, LayerNorm~\cite{ba2016}, GroupNorm~\cite{wu2018}, and weight normalization~\cite{salimans2016} constrain activation statistics but do not enforce energy conservation across module boundaries. The critical difference is causality: normalization sets output statistics independently of input; conservation ties output energy to input energy.

\textbf{Robustness methods.} Adversarial training~\cite{madry2018}, certified defenses, Lipschitz-constrained networks, and spectral normalization~\cite{miyato2018} improve robustness through various mechanisms. Conservation provides a complementary approach with a deterministic energy bound (Theorem~\ref{thm:attenuation}).

\textbf{Modular and compositional networks.} Mixture-of-experts~\cite{shazeer2017}, routing networks, and neural module networks~\cite{andreas2016} address \textit{which} modules to compose. Our work addresses \textit{how} to prevent errors from propagating between composed modules, a problem that has limited the adoption of modular sense-plan-act architectures in robotics~\cite{murphy2019, brooks1986, luo2024}.

\textbf{Energy-aware training.} Per-neuron energy costs and penalties have been explored for model efficiency on constrained hardware~\cite{schwartz2020}. These works \textit{minimize} energy for efficiency. We \textit{conserve} energy for robustness: an opposite mechanism producing an opposite effect, as demonstrated experimentally.

% ============================================================
\section{Conclusion}
% ============================================================

We have introduced energy conservation as a hard physical constraint for modular neural network pipelines. The mechanism adds zero learnable parameters, requires a single scaling operation per module boundary, and provides substantial, statistically significant noise attenuation.

Three experimental findings establish the approach, with a fourth localizing its mechanism and a fifth validating it on a real modular pipeline:

\begin{enumerate}[leftmargin=1.5em, itemsep=0.3em]
    \item \textbf{Statistical robustness:} $+42.3\%$ accuracy retention at $\sigma = 0.2$ over 5 random seeds ($p < 0.001$). Energy penalty produces the opposite effect ($-5.8\%$). These are fundamentally different mechanisms despite involving the same quantity.

    \item \textbf{Depth invariance:} 93.3\% retained at depth 2, 93.3\% at depth 5, with noise at every boundary. Modular pipelines can be deepened without robustness cost.

    \item \textbf{Noise generalization:} $+40$\textendash $45\%$ advantage on Gaussian and systematic bias noise, $+4.8\%$ on adversarial noise, $-0.3\%$ on dropout. The dropout limitation follows from physics: conservation attenuates energy-additive noise, not information-destroying noise.

    \item \textbf{Mechanism localization:} On ResNet-18, the conservation advantage scales inversely with the architecture's intrinsic normalization. With BatchNorm, the gap at $\sigma = 0.2$ is $+0.3$ pp. Without BatchNorm, it widens to $+26.2$ pp at $\sigma = 0.2$ and $+58.0$ pp at $\sigma = 0.5$. Conservation is the per-sample, training-free analogue of the regularization that BatchNorm provides intra-architecturally.

    \item \textbf{Modular pipeline validation:} On a perception-to-control pipeline (MuJoCo physics, Franka Panda), reproduced across three independent machines (90 trials per cell), conservation provides a $+18.9$ pp average advantage on monocular-depth-style noise and a neutral effect on isotropic Gaussian noise in 3D. The noise-type specificity is itself a finding: conservation's advantage tracks the structure of the dominant noise mode.
\end{enumerate}

The consistency of the depth-drift result across machines, the principled nature of both the limitations (information-destroying noise, high-noise SNR collapse) and the noise-type specificity, and the simplicity of the mechanism suggest that energy conservation is a targeted tool for robust modular neural computation. Future work will replace the oracle perception model with a full computer-vision pipeline, port the experiment to physical hardware, and extend the validation to transformer architectures and larger datasets.

% ============================================================
% REFERENCES
% ============================================================

\end{document}